\documentclass[runningheads,a4paper]{llncs}
\usepackage[english]{babel}
\usepackage{amssymb}
\usepackage{graphicx}
\usepackage{csquotes}
\usepackage[acronym,toc, xindy]{glossaries}
\usepackage{url}
\usepackage{epstopdf}
\usepackage[noend]{algpseudocode}
\usepackage{algorithm}
\usepackage{algorithmicx}

\urldef{\mailsa}\path|{l.johard, v.rivera, m.mazzara, j.lee}@innopolis.ru|

\begin{document}

\mainmatter  

\title{Self-adaptive node-based PCA encodings}


%
%
\author{Leonard Johard \and Victor Rivera \and Manuel Mazzara \and JooYoung Lee}

\authorrunning{}
%

\institute{Innopolis University\\
1, Universitetskaya Str., Innopolis, Russia, 420500\\
\mailsa\\
\url{https://www.university.innopolis.ru}}


%
%

\maketitle
\abstract{In this paper we propose an algorithm, Simple Hebbian PCA, and prove that it is able to calculate the principal component analysis (PCA) in a distributed fashion across nodes. It simplifies existing network structures by removing intralayer weights, essentially cutting the number of weights that need to be trained in half. 
}

\section{Introduction}
Innovative engineering always looks for smart solutions that can be deployed on the territory for both civil and military applications and, at the same time, aims at creating adequate instruments to support developers all along the development process so that correct software can be deployed. Modern technological solutions imply a vast use of sensors to monitor an equipped area and collect data, which will be then mined and analyzed for specific purposes. Classic examples are smart buildings and smart cities \cite{Salikhov2016a,Salikhov2016b}. 

Sensor integration across multiple platforms can generate vast amounts of data that need to be analyzed in real-time both by algorithmic means and by human operators. The nature of this information is unpredictable a priori, given that sensors are likely to encounter both naturally variable conditions in the field and disinformation attempts targeting the network protocols. 

This information needs to be transmitted through a distributed combat cloud with variable but limited bandwidth available at each node. Furthermore, the protocol has to be resistant to multiple node failures. 

The scaling of the information distribution also benefits from a pure feedforward nature, since the need for bidirectional communication scales poorly with the likely network latency and information loss, both of which are considerable in practical scenarios \cite{soyata2012combat,kruger2014implementing}. This requirement puts our desired adaptive system into the wider framework of recent highly scalable feedforward algorithms that have been inspired by biology \cite{johard2014connectionist}.

\section{Linear sensor encodings}
Linear encoding of sensor information has a disadvantage in that it cannot make certain optimizations, such as highly efficient Hoffman-like encodings on the bit level. On the other hand, it is very robust when it encodes continuous data, since it is isometric. This means that we will not see large disruptions in the sample distance and makes linear encodings highly suitable for later machine learning analysis and human observation. This isometry also makes the encoding resistant to noisy data transfers, which is essential in order to achieve efficient network scaling of real-time data.

The advantage of a possible non-linear encoding is further diminished if we consider uncertainty in our data distribution estimate. A small error in our knowledge can cause a large inefficiency in the encoding and large losses for lossy compression. For linear encodings all these aspects are limited, especially considering the easy use of regularization methods.

The advantage of linear encodings is that they possess a particular set of series of useful properties. To start with, if our hidden layer Y forms an orthonormal basis of the input layer we can represent the encoding as :

\begin{equation}I_{tot} = I_1 + I_2 ... + I_n + e^2 \end{equation}
Here $I_{tot}$ is the variance $\sum\limits_{i}(X_i^2)$ in the input space, $I_n$ is the variance of each component of $Y$ and $e^2$ is the squared error of the encoding. This is obvious if we add the excluded variables $y_{n+1} ... y_m$ and consider a single data point:
\begin{equation}x_i^2 = y_1^2 + y_2^2 ...  + y_n^2 + y_{n+1}^2 ... + y_m^2\end{equation}
and 
\begin{equation}y_{n+1} ... + y_m = e_i^2\end{equation}
where $e_i$ is the error for data point I . Summing both sides and dividing by number of data points and we get:
\begin{equation}var(I) = var(y_1) + ...  + var(y_n) + e^2\end{equation}

\section{PCA in networks}
The problem of encoding in node networks is usually considered from the perspective of neural networks. We will keep this terminology to retain the vocabulary predominant in literature. A recent review of current algorithms for performing principal component analysis (PCA) in a node network or neural network is \cite{qiu2012neural}. We will proceed here with deriving PCA in linear neural networks using a new simple notation, that we will later use to illustrate the new algorithms.

Assume inputs are normalized so that they have zero mean. In this case, each output $y_i$ can be described as $y_i = X^Tw$, where $x$ is the input vector and $w$ is the weights of the neuron and $i$ is the index of the input in the training data. The outputs form a basis of the input space and  if $\left \|  w_i \right \| = 1$ and $w_i^Tw_j = 0$ for all $i, j$, then the basis is orthonormal.

Let us first consider the simple case of a single neuron. We would like to maximize the variance on training data $E \left [ \frac{y^2}{2}\right ]$, where we define $y = X^Tw$, given an input matrix formed by placing column wise listing of all the presented inputs $X = [x_1, x_2 ... ]$ with the constraint $\left \|  w \right \| = 1$. Expanding:
\begin{equation}E \left [ \frac{y^2}{2} \right ] =  (X^Tw) ^T(X^Tw) = w^TXX^T w = w^TCw\end{equation}
where $C$ is the correlation matrix of our data, using the assumtions that inputs have zero mean. The derivative $\frac{\partial }{\partial w} E \left [ \frac{y^2}{2}\right ]$ is given by
\begin{equation}\frac{\partial }{\partial w} \frac{w^TCw}{2}  = XX^Tw = Xy\end{equation}
Note that the vector above describes the gradient of the variance in weight space. Taking a step of fixed length along the positive direction of this gradient derives the Hebb rule:
\begin{equation}w = w + \Delta w\end{equation}
\begin{equation}\Delta w = \eta Xy \end{equation}
Since we have no restrictions on the length of our weight vector, this will always have a component in the positive direction of $w$.  This unlimited growth of the weigth vector is easily limited by normalizing the weight vector $w$ after each step by dividing by length, $w_{norm} = \frac{w}{\left \|w\right \|}$.
If we thus restrict our weight vector to unit length and note that C is a positive semidefinite matrix we end up with a semi-definite programming problem:
\begin{equation}max\; w^TCw\end{equation}
subject to
\begin{equation}w^Tw = 1\end{equation}
It is thus guaranteed, except if we start at an eigenvector, that gradient ascent converges to the global maximum, i.e. the largest principal component.
Alternatives to weight normalization is to subtract the $e_w$ component of the gradient explicitly, where $e_w$ is the unit vector in the direction of $w$.  In this case we would calculate:
\begin{equation}\frac{\partial }{\partial w} (\frac{y^2}{2}) - (\frac{\partial }{\partial w} (\frac{y^2}{2}) \cdot e_w) e_w \end{equation}
For a step-based gradient ascent we can not assume $\left \|  w_i \right \|$ will be kept constant in the step direction. We can instead use the closely related
\begin{equation}\frac{\partial }{\partial w}(\frac{y^2}{2})  - w^Tw(\frac{\partial }{\partial w} (\frac{y^2}{2})  \cdot e_w) e_w\end{equation}
The difference is that the $w$ overcompensates for the $e_w$ component if $w^Tw > 1$ and vice versa. This essentially means that $\left \|  w_i \right \|$ will converge towards 1.
\begin{equation}\Delta w = \eta (Xy – wy^Ty) = \eta (XX^Tw - w^TXX^Tww) \end{equation}
\begin{equation}= \eta (Cw  -  w^TCww)\end{equation}
The derivative orthogonal to the constraint can be calculated as follows:
\begin{equation}\Delta w \cdot e_w = \eta w^T(Cw  - w^TCww)  = \eta(w^TCw - w^Tw^TCww) \end{equation}
This means that we have an optimum if \begin{equation}\label{eq:maxoja}((w^TCw)  - ww^T(w^TCw)) = 0\end{equation} 
 Since $w^TCw$ is a scalar, $w$ is an eigenvector of $C$ with eigenvalue $w^TCw$.  
Equation \ref{eq:maxoja} gives that $w^Tw = 1$

This is learning algorithm is equivalent to Oja's rule \cite{oja1982simplified}.

\subsection{Generalized Hebbian Algorithm}

The idea behind the generalized Hebbian algorithm (GHA) \cite{sanger1989} is as follows:\\ 

\begin{itemize}
\item[] 1. Use Oja's rule to get $w_i$
\item[] 2. Use deflation to remove variance along $e_{w_i}$
\item[] 3. i := i +1
\item[] 4. Go to step 1

\end{itemize}
Subtraction of the $w$-dimension projects the space into the subspace spanned by the remaining principal components. The target function $\frac{y(v_i)^2}{2}$ for all eigenvectors $v_i$ not eliminated by this projection, while $\frac{y(w)^2}{2}$ = 0 in the eliminated direction $w$.  Repeating the algorithm after this step guarantees that we will get the largest remaining component at each step.
The GHA requires several steps to calculate the smaller components and uses a specialized architecture.
The signal needs to pass through $2(n-1)$ neurons in order to calculate the $n$-th principal component and uses two different types of neurons to achieve this.


We define information as the square variance of the transmitted signal and seek encodings that will attempt to maximize the transmitted information. In other words, the total transmitted variance by a linear transform is equal to the variance of data projected onto a subspace of the original input space. The variance in this subspace plus the square error of our reconstruction is equal to the variance of the input.

Summarizing, minimizing the reconstruction error of our encoding  is equivalent to maximizing the variance of the output. This is complementary and not antagonistic to the concept of sparse encodings disentangling the factors of variation \cite{bengio2013representation}.

\subsection{Distributed PCA}
Principal component analysis is the optimal linear encoding minimizing the reconstruction error, but still leaves room open for improvement. Can we do better? In PCA, as much as information as possible is put in each consecutive component. This leaves the encoding vulnerable to the loss of a node or neuron, potentially losing a majority of the information as a result.  

The PCA subspace remains the optimal subspace in this sense regardless the vectors chosen to span it. Thus, any rotation the orthonormal basis is also an optimal linear encoding. 

\begin{theorem}
There exists an encoding of the PCA space such that the information along each component is equivalent, $I_n = I_m,  \forall n, m$. This encoding minimizes the maximum possible error of any combination $n-1$ components.
\end{theorem}

\begin{proof}

Starting from the eigenvectors $v_i$, we can rotate any pair of vectors in the plane spanned by these vectors. As long as orthogonality is preserved, the sum of the variance in the dimensions spanned by these vectors is constant. Expressed as an average:
\begin{equation}\sum_i I_i = \sum_i k \end{equation}

Now for this to be true and if not all variances $I_i$ are identical there has to exist a pair of indices $i$ and $j$ such that $I_i < k < I_j$. We can then find a rotation in the plane spanned by these vectors such that $I_i = k$.\\

This simple algorithm can be repeated until $\forall i: I_i = k$.\\

In matrix form this can be formulated as:\\
\begin{equation}diag(WCW^T) = kI\end{equation}
Orthonormal basis:
\begin{equation}WW^T = I\end{equation}
$=>$
\begin{equation}diag(WCW^T) - WW^T = (c - 1)I\end{equation}
\begin{equation}diag((WCW - W)W^T) = (c – 1)I\end{equation}
\begin{equation}W(C-I)W^T = (c-1)I \forall I\end{equation}

\end{proof}

This seems like a promising candidate for a robust linear encoding and future work will further explore the possibility for calculating these using Hebbian algorithms. For the moment, we will instead focus on the eigenvectors to the correlation matrix used in regular PCA.

 \subsection{Simple Hebbian PCA}
We propose a new method for calculating the PCA encoding $X \rightarrow Y$ in a single time step and using a single weight matrix $W$. 

For use in distributed transmission systems an ideal algorithm should process only local and explicitly transmitted information in terms of $X$ and $Y$ from its neighbors. In other words, each node possesses knowledge about its neighbors' transmission signal, but not their weights or other information. The Simple Hebbian PCA is described in pseudocode in algorithm \ref{ASHP}.

\begin{algorithm}
\caption{$ASHP$}\label{ASHP}
\begin{algorithmic}[]
\Require Initialized weight vector $w_i$ 
\Require Input matrix $X$ 
\Require Number of iterations $T$
\Require Number of nodes $N$
\Require Step size $\eta$
\For{$t \gets 1$ to $T$}
\For{$i \gets 1$ to $N$}
\State $y_i \gets Xw_i$
\EndFor
\For{$i \gets 1$ to $N$}
\State $\displaystyle w_i \gets w_i + \eta (Xy_i -  \sum_{j \mathop = 1}^i \frac{Xy_j y_i^Ty_j}{y_j^Ty_j})$  
\State $w_i \gets \frac{w_i}{w_i^Tw_i}$
\EndFor
\EndFor
\end{algorithmic}
\end{algorithm}

 \subsubsection{Convergence property}
The first principal component can be calculated as $\Delta w = Xy$. This step is equivalent to Oja's algorithm.
Let $n$ be the index of the largest eigenvector calculated so far. The known eigenvectors $v_1, v_2 ... v_n$ of the correlation matrix $C$ have corresponding eigenvalues $ \lambda_1, \lambda_2 ... \lambda_n$.
We can now calculate component $v_{n +  1}$.\\
\begin{lemma}
\begin{equation}\label{eq:hpcas1}f_n(w) = \frac{y^2}{2} - \sum_{i=1}^{n}\frac{y^Ty_iy^Ty_i}{2 \lambda_i^2}\end{equation}
 has for $w^Tw = 1$ a maximum at$ w = v_n + 1$, where $y = w^TX$ and $y_n = v_n^TX$\\
\end{lemma}
\begin{proof}
We have an optimum if the gradient lies in the direction of the constraint $w^Tw = 1$, i.e. \begin{equation}\frac{\partial}{\partial w} f_n = k w\end{equation} for some constant k.
\begin{equation}\frac{\partial}{\partial w}f_n = Cw - \sum_{i=1}^{n}\frac{Cv_iw^TCv_i}{\lambda_i^2}\end{equation}
Which further simplifies to
 \begin{equation}\label{eq:hpcas2} (C - \sum_{i=1}^{n}\frac{ Cv_iv_i^TC}{\lambda_i^2 })w = C_nw\end{equation}
where we define $C_n$ as the resulting matrix of the above parenthesis.

To reach an optimum we seek
\begin{equation}w^TC_n = cw\end{equation}
where $c$ is some scalar.

Our optimal solution has the following properties:\\\\
1. Assume $w = v_i, i \leq  n$ :\\
Substituting $w = v_i$ in \ref{eq:hpcas2} we get
\begin{equation}\label{eq:hpca1} \frac{\partial}{\partial w} f_n(v_i) = \lambda_i v_i  - \lambda_i v_i = 0 \cdot v_i\end{equation}
then $v_i$ is an eigenvector of $\frac{\partial}{\partial w} f_n$ with eigenvalue 0.\\\\
2. Assume $w = v_i$ of $C$, $i > n$:\\
Substituting $w = v_i$ in \ref{eq:hpcas2} we get
\begin{equation}\label{eq:hpca2}\frac{\partial}{\partial w} f(v_i) = Cw = \lambda_i w\end{equation}
then $v_i$ is an eigenvector of $\frac{\partial}{\partial w} f_n$ with eigenvalue $\lambda_i$.\\\\
\end{proof}
$ C $ is symmetric and real. Hence, the eigenvectors $ v_1 ... v_n $ span the space $ \mathbb{R}^n $. $C_n$ is a sum of symmetric matrices. Consequently $C_n$ is symmetric with the same number of orthogonal eigenvectors. As we see in equations \ref{eq:hpca1} and \ref{eq:hpca2}, every eigenvector $v_i$ of $C$ is an eigenvector of $C_n$, with eigenvalue $\lambda_{n,i} = 0$ if $i \leq  n$ and $\lambda_{n,i} = \lambda_i$ if $i > n$. Since $\lambda_n$ are ordered by definition, $\lambda_{n+1}$ is the largest eigenvalue of $C_n+1$.

$C_n$ is symmetric with positive eigenvalues. As a result $C_n$ is positive semi-definite. For this reason the maximization problem 
\begin{equation}sup(w^TC_nw)\end{equation} \begin{equation} w^Tw = 1\end{equation} forms another convex optimization problem and gradient ascent will reach the global optimum, except if we start our ascent at an eigenvector where $\frac{\partial}{\partial w} f_n(v_i) = 0$. For random starting vectors the probability of this is zero.

The projection of the gradient onto the surface $w^Tw = 1$  created by weight normalization follows $\delta w \cdot \frac{\delta w}{w^Tw} > 0$, i.e. even for steps not in the actual direction of the unconstrained gradient the step lies in a direction of positive gradient.

This algorithm has some degree of similarity to several existing algorithms, namely the Rubner-Tavan PCA algorithm \cite{rubner1989self}, the APEX-algorithm \cite{kung1990neural} and their symmetric relatives \cite{pehlevan2015hebbian}. In contrast to these, we only require learning of a single set of weights $w$ per node and avoid the weight set $L$ for connections within each layer. 

\section{Conclusions}
We have proposed algorithm, Simple Hebbian PCA, and proof that it is able to calculate the PCA in a distributed fashion across nodes. It simplifies existing network structures by removing intralayer weights, essentially cutting the number of weights that need to be trained in half. 

This means that the proposed algorithm has an architecture that can be used to organize information flow with a minimum of communication overhead in distributed networks. It automatically adjusts itself in real-time so that the transmitted data covers the optimal subspace for reconstructing the original sensory data and is reasonably resistant to data corruption.

In future work we will provide empirical results of the convergence properties. We also seek to derive symmetric versions of our algorithm that uses the same learning algorithm for each node, or in an alternative formulation, that uses symmetric intralayer connections.

Eventually we also strive toward arguing for biological analogies of the proposed communication protocol as way of transmitting information in biological and neural networks.

\bibliographystyle{ieeetr}
\bibliography{references}

\end{document}